\documentclass[conference]{IEEEtran}

\usepackage{times}
\usepackage{latexsym}

\usepackage[T1]{fontenc}

\usepackage[utf8]{inputenc}

\usepackage{microtype}

\usepackage{inconsolata}

\usepackage{amsmath}
\usepackage{amssymb}
\usepackage{amsthm}
\usepackage{mathtools}
\usepackage{dsfont}
\usepackage{mathrsfs}  
\usepackage{breqn}     

\usepackage{algorithm}
\usepackage[noend]{algpseudocode}

\usepackage{graphicx}
\usepackage{subcaption}
\usepackage{url}
\usepackage{xcolor}
\usepackage{caption}
\usepackage{float}
\usepackage{booktabs}  

\usepackage{hyperref}

\usepackage[mathscr]{euscript}
\usepackage{dutchcal}

\newtheorem{theorem}{Theorem}

\DeclareMathOperator*{\argmin}{\arg\!\min}

\DeclarePairedDelimiterX{\expectarg}[1]{[}{]}{%
  \ifnum\currentgrouptype=16 \else\begingroup\fi
  \activatebar#1
  \ifnum\currentgrouptype=16 \else\endgroup\fi
}
\newcommand{\activatebar}{%
  \begingroup\lccode`\~=`\|
  \lowercase{\endgroup\let~}\innermid 
  \mathcode`|=\string"8000
}
\newcommand{\innermid}{\nonscript\;\delimsize\vert\nonscript\;}

\algnewcommand{\Inputs}[1]{%
  \State \textbf{Inputs:}
  \Statex \hspace*{\algorithmicindent}\parbox[t]{.8\linewidth}{\raggedright #1}
}
\algnewcommand{\Initialize}[1]{%
  \State \textbf{Initialize:}
  \Statex \hspace*{\algorithmicindent}\parbox[t]{.8\linewidth}{\raggedright #1}
}
\algnewcommand{\Data}[1]{%
  \State \textbf{Data:}
  \Statex \hspace*{\algorithmicindent}\parbox[t]{.8\linewidth}{\raggedright #1}
}

\title{Tricks and Plug-ins for Gradient Boosting with Transformers}

\author{
Biyi Fang\textsuperscript{1} \quad
Truong Vo\textsuperscript{1} \quad
Jean Utke\textsuperscript{2} \quad
Diego Klabjan\textsuperscript{1} \\
\textsuperscript{1}Northwestern University \quad
\textsuperscript{2}Allstate \\
{\tt\small biyifang2021@u.northwestern.edu} \quad
{\tt\small truongvo2025@u.northwestern.edu} \\
{\tt\small jutke@allstate.com} \quad
{\tt\small d-klabjan@northwestern.edu}
}

\begin{document}
\maketitle
\begin{abstract}
Transformers, powered by multi-head attention and deep processing layers, have become the foundation of modern NLP models, achieving state-of-the-art results across diverse tasks. However, their success often comes at the cost of high computational demand and prolonged training times, particularly due to the need for extensive architecture and hyperparameter tuning. In this paper, we present a novel family of algorithms that combine boosting techniques with transformer architectures to enhance training efficiency and model performance. Our approach introduces a subgrid selection mechanism and an importance sampling strategy, which leverage boosting weights and integrate them into the transformer training pipeline via a least squares objective. This framework mitigates the reliance on manual architecture design and accelerates convergence, while preserving the expressive power of transformers. Empirical evaluations on multiple fine-grained language classification benchmarks reveal that our boosted transformer variants outperform standard transformers in both accuracy and training efficiency.
\end{abstract}
\section{Introduction}

Transformers, such as BERT, GPT, RoBERTa, and T5, have demonstrated outstanding performance across a wide range of natural language processing (NLP) tasks, including sentiment analysis \cite{batra2021bert_sentiment, munikar2019finegrained}, question answering, and text generation \cite{raffel2020exploring, brown2020language}. 
These models have also been successfully applied in speech-related NLP tasks such as speech recognition \cite{Karita2019ImprovingTE,Yeh2019TransformerTransducerES,Dong2018SpeechTransformerAN} and machine translation \cite{Li2019NeuralSS,Gangi2019AdaptingTT,Zhou2018SyllableBasedSS}. Despite their success, identifying an optimal architecture or fine-tuning strategy for a given task remains a challenging problem, as task-specific text representations vary significantly.

To address these challenges, a growing body of research has focused on improving the efficiency and adaptability of transformer models. Neural Architecture Search (NAS) methods such as AutoFormer \cite{chen2021autoformer} and Primer \cite{so2021primer} have been proposed to automate transformer design, particularly for efficient and lightweight architectures. While effective, these methods still require substantial computational resources and extensive GPU usage. An alternative line of work explores ensemble techniques, especially boosting, which offer strong theoretical guarantees and practical performance benefits. For example, AdaMix \cite{wang2022adamix} introduces a mixture-of-adapter approach that captures boosting effects while maintaining parameter efficiency. However, existing techniques still face challenges related to memory usage and scalability—particularly when large pre-trained models are used as base learners. Limited work has explored the design of boosting strategies under partial data regimes, where each weak learner is trained on a distinct subset of the data. Exploring this direction could improve training efficiency, promote model diversity, and enhance generalization—especially in low-resource or fine-grained NLP tasks.

In this paper, we propose \textbf{BoostTransformer}, a novel family of boosting algorithms designed for sequence data that synergistically combines the principles of boosting and transformers. To promote diversity among weak learners, we adopt a feature subsampling strategy inspired by random forests. Adapting this approach to unstructured sequential data requires several novel design choices. Additionally, we integrate importance sampling by assigning probabilities to individual samples, guiding the model to focus on more informative examples. BoostTransformer further incorporates boosting weights into the transformer architecture and optimizes each learner using a least squares objective, enabling robust and efficient sequential learning. 

To further enhance efficiency and reduce training overhead, we introduce \textbf{Subsequence BoostTransformer}, an extension that eliminates the need for full-sequence data input when training each weak learner. In Subsequence BoostTransformer, important tokens are selectively extracted from the input sequence based on attention distributions \cite{Vaswani2017AttentionIA}. While this selection may omit direct dependencies between consecutive tokens, it preserves the most informative content, allowing the model to learn effectively from compressed input. As a result, Subsequence BoostTransformer achieves higher accuracy with lower computational cost compared to the former BoostTransformer.

Moreover, motivated by the observation that BoostTransformer tends to overfit early during training, we propose a third variant, \textbf{Importance-sampling BoostTransformer}, which explicitly addresses this issue by combining boosting with a principled importance sampling strategy. This algorithm first computes a probability distribution over all training samples based on their residuals. At each boosting iteration, it selects a subset of samples according to this distribution and trains a weak learner on the selected data. This approach not only delays overfitting but also improves generalization accuracy and significantly reduces runtime. We provide a formal proof showing that the optimal importance-sampling distribution is proportional to the norm of the residuals.

Finally, we conduct comprehensive computational experiments demonstrating the advantages of our proposed methods. BoostTransformer consistently outperforms standard transformers in accuracy and stability. Both Subsequence BoostTransformer and Importance-sampling BoostTransformer offer additional gains by reducing training time while maintaining or improving predictive performance. On the IMDB, Yelp, and Amazon datasets, BoostTransformer, Subsequence BoostTransformer, and Importance-sampling BoostTransformer achieve average accuracy improvements of 0.87\%, 0.55\%, and 0.79\%, respectively, over standard transformers. Notably, Subsequence BoostTransformer and Importance-sampling BoostTransformer require only two-thirds and one-half of the training time, respectively.

In summary, we make the following contributions.
\begin{itemize}
    \item We introduce BoostTransformer integrating transformer architectures with feature subsampling and importance sampling.
    \item We develop two enhanced variants: Subsequence BoostTransformer and Importance-sampling BoostTransformer.
    \item We conduct comprehensive experiments, demonstrating that all proposed methods outperform standard transformers in accuracy and stability. 

\end{itemize}

\section{Related Work}
In this section, we review several variations of Gradient Boosting Machines (GBMs) that are most relevant to our proposed algorithms, along with two important enhancements: subgrid token selection and importance sampling.

Recurrent neural networks (RNNs), long short-term memory networks (LSTMs), and transformers have become foundational models for sequence modeling and transduction tasks such as language modeling and machine translation \cite{Bahdanau2015NeuralMT, Sutskever2014SequenceTS, Vaswani2017AttentionIA}. Some prior works have explored hybrid models that combine boosting with RNNs or LSTMs. For instance, Chen and Lundberg \cite{Chen2018HybridGB} proposed a hybrid model in which LSTMs are used for supervised representation learning, and the resulting features are fed into a gradient boosting tree model (XGBoost) for prediction. However, the feature selection process is not influenced by the boosting objective, resulting in a disconnect between the representation and the downstream model. Another attempt, by Assaad et al. \cite{Assaad2008ANB}, introduced a boosting-based regression framework using RNNs as base learners. Their approach emphasizes different time points by training base learners on selected subsets of the time series. While these works provide preliminary connections between boosting and sequence models, the integration of boosting with transformer architectures remains largely unexplored. Additionally, while attention mechanisms in transformers have been studied extensively to interpret model behavior \cite{Clark2019WhatDB}, the use of attention distributions to guide token selection—such as in our Subsequence BoostTransformer has not been systematically investigated.

The second enhancement we incorporate is importance sampling, which has been shown to improve convergence and efficiency in various learning paradigms. Prior work has explored importance sampling in the context of stochastic gradient descent (SGD) \cite{Needell2014StochasticGD, Zhao2015StochasticOW}, deep learning \cite{Katharopoulos2018NotAS}, and minibatch sampling \cite{Csiba2018ImportanceSF}. These studies demonstrate that prioritizing samples based on their contribution to the loss or gradient norm can accelerate training and reduce variance. However, despite its success in SGD-based frameworks, importance sampling has not yet been generalized to boosting algorithms, leaving a promising direction that our work begins to explore.

\section{Algorithms}
In this section, we propose three algorithms combining boosting and transformers from different perspectives. We assume a BERT-like bidirectional transformer classifier \cite{Devlin2019BERTPO,liu2019roberta}. The first token of each sequence is a special classification token, and the corresponding final hidden state output of this token is used as the aggregated representation for the classification.

\subsection{BoostTransformer}
Inspired by BoostCNN \cite{Brahimi2019BoostedCN}, we propose BoostTransformer which combines boosting and transformers for a sequence classification problem. We are given a sample $x_i\in\mathcal{X}$, which contains a sequence of tokens, and its class label $z_i\in\left\{1,2,\cdots,M\right\}$. The risk function, the functional gradient and the optimal boosting coefficient $\alpha^t$ are exactly the same as those in (\ref{risk function}), (\ref{functional gradient}), and (\ref{boost parameter}) (Appendix \hyperref[sec:background]{A}). The algorithm follows standard gradient boosting machine.
\begin{algorithm}[H]
\footnotesize

  \caption{BoostTransformer}
  \label{alg:boostTransformer}
  \begin{algorithmic}[1]
  \Inputs{number of classes $M$, number of boosting iterations $N_b$, shrinkage parameter $\nu$, dataset $\mathcal{D}=\left\{(x_1,z_1),\cdots,(x_n,z_n)\right\}$ where $x_i$ is a token sequence and $z_i \in \{1,\cdots,M\}$ is the class label}
  
  \Initialize{Initialize prediction function $f(x) = \mathbf{0} \in \mathbb{R}^M$}

  \For{$t = 1,2,\cdots,N_b$}
    \State Compute example weights $w(x_i, z_i)$ using $f(x_i)$ 
    
    \Comment{see Eq.~(\ref{weight update})}
    \State Train a transformer-based weak learner $g_t^*$ 
    
    \Comment{see Eq.~(\ref{weak learner train})}
    \State Compute the optimal step size $\alpha_t$ using Eq.~(\ref{boost parameter})
    \State Update ensemble prediction: 
    
    $f(x) \gets f(x) + \nu \alpha_t g_t^*(x)$
  \EndFor

  \State \Return final classifier $f(x)$
  \end{algorithmic}
\end{algorithm}

\subsection{Subsequence BoostTransformer}
Combining the subgrid trick and BoostTransformer means applying the subgrid trick to each weak learner in BoostTransformer. Different from deep CNNs, transformers are able to deal with sequences of any length, thus, there is no issue when transferring information from the current weak learner to the succeeding weak learner. We denote $g_0$ as the basic weak learner, which deals with the whole dataset, and all the succeeding $g_t$'s as the additive weak learners. Moreover, subsequence BoostTransformer defines an importance index for each token $\mathcal{w}$ in the vocabulary based on the attention distribution. More precisely, the importance value of token $\mathcal{w}$ is computed by adding two parts; the first part is the importance of the token $\mathcal{w}$ itself, and the second part is the importance of token $\mathcal{w}$ to the remaining tokens in the same sample. In an $L$-layer transformer for a sequence $x$ of length $s$ (following \cite{Devlin2019BERTPO} we assume that the first token in $x$ is a placeholder, which indicates that the corresponding token in the final layer is used as the embedding for classification), and positions $1\leq i,j\leq s$, and layer $k$ for $1\leq k\leq L$, let the attention from position $i$ to position $j$ between layer $k-1$ and $k$ be denoted by $a(i,j;k;x)$. We have $\sum_{j=1}^s a(j,i;k;x)=1$ for every $i,k,x$. Then, given a transformer with $L$ layers, the self-importance of token $\mathcal{w}$ in position $p$ in a sample $x_i$ is
\begin{align}
\label{importance_self}
I^S(\mathcal{w},x_i)&=\left(\prod_{k=1}^{L-1} a(p,p;k;x_i)\right)\cdot a(p,1;L;x_i)\nonumber\\
&\approx a(p,1;L;x_i),
\end{align}
The importance of token $\mathcal{w}$ to others is 
{\footnotesize
\begin{align}
\label{importance_rest}
I^R(\mathcal{w},x_i)=\left[\prod_{k=1}^{L-1}\max_{j,j\neq p} a(p_{k-1},j;k;x_i)\right]\cdot a(p_{L-1},1;L;x_i),
\end{align}
}
where $p_{k-1}=\mathrm{argmax}_{j,j\neq p}a(p_{k-2},j;k-1;x_i)$  for $k=2,3,\cdots, L-1$, and $p_0=p$. The first term computes the product of the maximum attention values through the path which does not contain $p$ until the second to last layer. For the second term, as it has been shown in \cite{Devlin2019BERTPO}, the classification layer only takes the $1$st position of the last transformer layer which is corresponding to the classification token, therefore, the formula in (\ref{importance_rest}) does not check all possible attention distributions; instead, it counts the attention value from the position $p_{L-1}$ to the $1$st position in the last transformer layer directly. After the aforementioned importance values are computed, the importance value of the vocabulary word $\hat{\mathcal{w}}$ is
\begin{align}
\label{importance_agg}
    I(\hat{\mathcal{w}})=\sum_{
    \begin{matrix}
x_i,\mathcal{w}\in x_i \\
\mathcal{w} = \hat{\mathcal{w}}
\end{matrix}} \left(I^S(\mathcal{w},x_i)+ I^R(\mathcal{w},x_i)\right).
\end{align}
Once the important tokens are selected, the algorithm constructs a filtered version of each input sequence, denoted $x_i^t$, containing only the tokens in the selected vocabulary subset $V_t$. The corresponding weak learner $g_t$ is then trained to fit the residuals using the reduced input. The objective at iteration $t$ becomes:
\begin{align}
\label{subgrid weak learner train}
  \mathcal{L}(w,g)=\sum_{(x_i,z_i)\in\mathcal{D}}\left\|g(x_i^t)-w(x_i,z_i)\right\|^2, 
\end{align}
where $w(x_i, z_i)$ is the boosting weight derived from the functional gradient of the ensemble’s current prediction. The final boosted classifier aggregates the predictions of all weak learners trained on token-pruned subsequences:
\begin{align}
    f(x) = \sum_{t=1}^N \alpha_t g_t(x^t),
    \label{subgrid classifier}
\end{align}
where each $x^t$ denotes the subsequence of $x$ restricted to the token subset used at iteration $t$, and $\alpha_t$ is the step size selected to minimize the overall boosting loss.

Different from standard BoostTransformer, subsequence BoostTransformer first reviews the whole dataset in steps~\ref{subT:basic start}-\ref{subT:basic end} and generates the basic weak learner $g_0^*$. Once the basic weak learner is created, in each iteration, subsequence BoostTransformer first updates the attention-based importance vector $I_{\mathcal{w}}$ for any $\mathcal{w}\in V_{t-1}$ in step~\ref{subT:generate matrix}, and selects $\sigma$ fraction of the tokens to form the vocabulary set $V_t$, and lastly constructs a new sample $x_i^t$ by deleting any tokens not in $V_t$ in step~\ref{subT:select subgrid}. After the new sample $x_i^t$ is constructed, subsequence BoostTransformer initializes the weights of the current transformer by using the weights in $g_{t-1}^*$ and trains the transformer with $x_i^t$ to minimize the squared error in (\ref{subgrid weak learner train}) in steps~\ref{subT:subgrid start}-\ref{subT:subgrid end}. Lastly, the algorithm finds the boosting coefficient $\alpha_t$ by minimizing (\ref{boost parameter}) in step~\ref{subT:coeff} and adds the additive weak learner to the ensemble in step~\ref{subT:subgrid update}.

\begin{algorithm}[H]
\footnotesize
  \caption{subsequence BoostTransformer}
  \label{alg:subBoostTrans}
  \begin{algorithmic}[1]
  \Inputs{number of classes $M$, number of boosting iterations $N_b$, shrinkage parameter $\nu$, dataset $\mathcal{D}=\left\{(x_1,z_1),\cdots,(x_n,z_n)\right\}$ where $z_i\in\left\{1,\cdots,M\right\}$ is the label of sample $x_i$, and $0<\sigma<1$}
    \Initialize{set $f(x)=\mathbf{0}\in\mathbb{R}^M$, $V_0=\left\{\mathcal{w}\vert \mathcal{w}\in x_i  \mathrm{\,\,for\,\, some\,\,}x_i\right\}$ }
    \State compute $w(x_i,z_i)$ for all $(x_i,z_i)$, using (\ref{weight update}) \label{subT:basic start}
    \State train a transformer $g_0^*$ to optimize (\ref{weak learner train})\label{subT:basic end}
     \State $f(x) = g_0^*$
    \For{t = $1,2,\cdots ,$ $N_b$ }
    \State update importance values $I_{\mathcal{w}}$ for $\mathcal{w}\in V_{t-1}$, using (\ref{importance_self}), (\ref{importance_rest}) and (\ref{importance_agg})\label{subT:generate matrix}
    \State form $V_t\subset V_{t-1}$ with $\frac{|V_t|}{|V_{t-1}|}\approx\sigma$ and $I(\mathcal{w})>I(\mathcal{w}') \forall \mathcal{w}\in V_t,\mathcal{w}' \in V_{t-1}\setminus V_t$   and form a new sample $x_i^t$ for each sample $i$\label{subT:select subgrid}
    \State compute $w(x_i,z_i)$ for all $i$, using (\ref{weight update}) and (\ref{subgrid classifier})\label{subT:subgrid start}
    \State train a transformer $g_t^*$ to optimize (\ref{subgrid weak learner train})\label{subT:subgrid end}
    \State find the optimal coefficient $\alpha_t$, using (\ref{boost parameter}) and (\ref{subgrid classifier})\label{subT:coeff}
     \State $f(x) = f(x) + \nu\alpha_t g_t^*$\label{subT:subgrid update}
      \EndFor
      \State \textbf{end for}
  \end{algorithmic}
\end{algorithm}

\subsection{Importance-sampling-based BoostTransformer}
Importance sampling, a strategy for preferential sampling of more important samples capable of accelerating the training process, has been well studied in stochastic gradient descent (SGD) \cite{Alain2015VarianceRI}. However, there is virtually no existing work combining the power of importance sampling with the strength of boosting. Motivated by the phenomenon that overfitting appears early in standard BoostTransformer, we propose importance-sampling-based BoostTransformer, which combines importance sampling and BoostTransformer. Importance-sampling-based BoostTransformer mimics importance sampling SGD by introducing a new loss function and computing a probability distribution for drawing samples. Similarly, importance-sampling-based BoostTransformer computes a probability distribution in each iteration, and draws a subset of samples to train the weak learner based on the distribution. The probability distribution is
\begin{align}
\label{importance sample distribution}
    P(I=i)=\frac{\left\|w(x_i,z_i)\right\|}{\sum_{(x_j,z_j)\in\mathcal{D}}\left\|w(x_j,z_j)\right\|},
\end{align}
which yields the new loss function for a subset of samples ${\mathcal{I}}$ to be
{\footnotesize
\begin{align}
\label{importance sample train loss}
    \bar{\mathcal{L}}_{\mathcal{I}}(w,g)=\sum_{(x_i,z_i)\in\mathcal{I}}\frac{1}{\left|\mathcal{D}\right|P(I=i)}\left\|g(x_i)-w(x_i,z_i)\right\|^2.
\end{align}
}
To any minimization algorithm one would typically use.We then apply any optimization algorithm with respect to (\ref{importance sample train loss}) (by further using mini-batches or importance sampling). The entire algorithm is exhibited in Algorithm \ref{alg:impBoostTrans}.
\begin{algorithm}[H]
\footnotesize
  \caption{importance-sampling-based BoostTransformer}
  \label{alg:impBoostTrans}
  \begin{algorithmic}[1]
  \Inputs{number of classes $M$, number of boosting iterations $N_b$, shrinkage parameter $\nu$, dataset $\mathcal{D}=\left\{(x_1,z_1),\cdots,(x_n,z_n)\right\}$ where $z_i\in\left\{1,\cdots,M\right\}$ is the label of sample $x_i$, and $0<\sigma<1$}
    \Initialize{set $f(x)=\mathbf{0}\in\mathbb{R}^M$}
    \State compute $w(x_i,z_i)$ for all $x_i$, using (\ref{weight update}) \label{imT:basic start}
    \State train a transformer $g_0^*$ to optimize (\ref{weak learner train})\label{imT:basic end}
     \State $f(x) = g_0^*$
    \For{t = $1,2,\cdots,$ $N_b$ }
    \State compute probability distribution $P_t$, using (\ref{importance sample distribution})\label{imT:importance sample}
        \State draw independently  $|\mathcal{I}^t|$ samples, which is $\sigma$ fraction of the samples, based on $P_t$\label{imT:select sample}
        \State compute $w(x_i,z_i)$ for $(x_i,z_i)\in \mathcal{I}^t$, using (\ref{weight update}) \label{imT:subgrid start}
    \State train a transformer $g_t^*$ to optimize (\ref{importance sample train loss}) on $\mathcal{I}^t$\label{imT:subgrid end}
    \State find the optimal coefficient $\alpha_t$, using (\ref{boost parameter}) on $\mathcal{I}^t$\label{imT:coeff}
     \State $f(x) = f(x) + \nu\alpha_t g_t^*$\label{imT:subgrid update}
      \EndFor
      \State \textbf{end for}
  \end{algorithmic}
\end{algorithm}
Importance-sampling-based BoostTransformer starts with learning the full-size dataset and training a basic weak learner in steps~\ref{imT:basic start}-\ref{imT:basic end}. In each iteration, the algorithm first computes the probability distribution $P_t$ in step~\ref{imT:importance sample} and selects a subset $\mathcal{I}^t$ of samples based on the distribution in step~\ref{imT:select sample}. Once the dataset is created, it computes the weights and trains a transformer by using the unbiased loss function (\ref{importance sample train loss}), following by finding an optimal boosting coefficient in steps~\ref{imT:subgrid start}-\ref{imT:subgrid update}.

Given current aggregated classifier $f_{t-1}=\sum_{i=1}^{t-1}g_{i}^*$, let us define the expected training progress attributable to iteration $t$ as
{\footnotesize
\begin{align*}
    \mathbb{E}_{P_{t}}\left[ \Delta^{(t)} \right]=&\left\|f_{t-1}-f^*\right\|^2-\mathbb{E}_{P_{t}}\left[\left.\left\| f_{t}-f^*\right\|^2 \right|\mathcal{F}^{t-1}\right].
\end{align*}
}
Here $f^*$ denotes the solution to (\ref{risk function}), and the expectation is taken over the probability distribution $P_t$, and $\mathcal{F}^{t-1}$ contains the whole history of the algorithm up until iterate $t-1$. We assume that gradient sampling is unbiased. Inspired by the work in \cite{Zhao2015StochasticOW}, we prove that the optimal probability distribution is proportional to the boosting weight at each iteration.
\begin{theorem}\label{thm:1}
In $\max_{P_{t}}$ $\mathbb{E}_{P_{t}}\left[ \Delta^{(t)}\right]$, the optimal distribution for importance-sampling-based BoostTransformer to select each sample $i$ is proportional to its “boosting weight norm:” , i.e. (\ref{importance sample distribution}).
\end{theorem}
\begin{proof}
See Appendix \hyperref[pf:thm1]{B}.
\end{proof}
Based on the fact that the new loss function with respect to the probability distribution is unbiased, we discover that maximizing the improvement of the boosting algorithm is equivalent to minimizing the functional gradient variance. By applying Jensen's inequality, the optimal probability distribution is essentially proportional to the boosting weights, which are easy to obtain, in boosting algorithms.  

\section{Experimental Study}
In this section, we present an empirical evaluation comparing the performance of four models: the standard Transformer, BoostTransformer, Subsequence BoostTransformer, and Importance-sampling BoostTransformer. Our goal is to assess the individual and combined contributions of three key components: the boosting framework, the subgrid (subsequence) token selection strategy, and the importance sampling mechanism. These components are evaluated in terms of classification accuracy, generalization behavior, training stability, and computational efficiency. To ensure fair comparisons, we maintain consistent model configurations and training settings across all experiments unless otherwise specified. The training is conducted on an NVIDIA Titan XP GPU, with hyperparameters tuned individually for each method using validation performance. We report results on three widely used benchmark datasets—IMDB, Yelp, and Amazon—which span different domains and text lengths, allowing us to assess the robustness of the proposed methods across diverse sequence modeling tasks.

\begin{figure*}[!t]
\begin{minipage}[t]{0.5\textwidth}
  \centering
  \includegraphics[width=\linewidth]{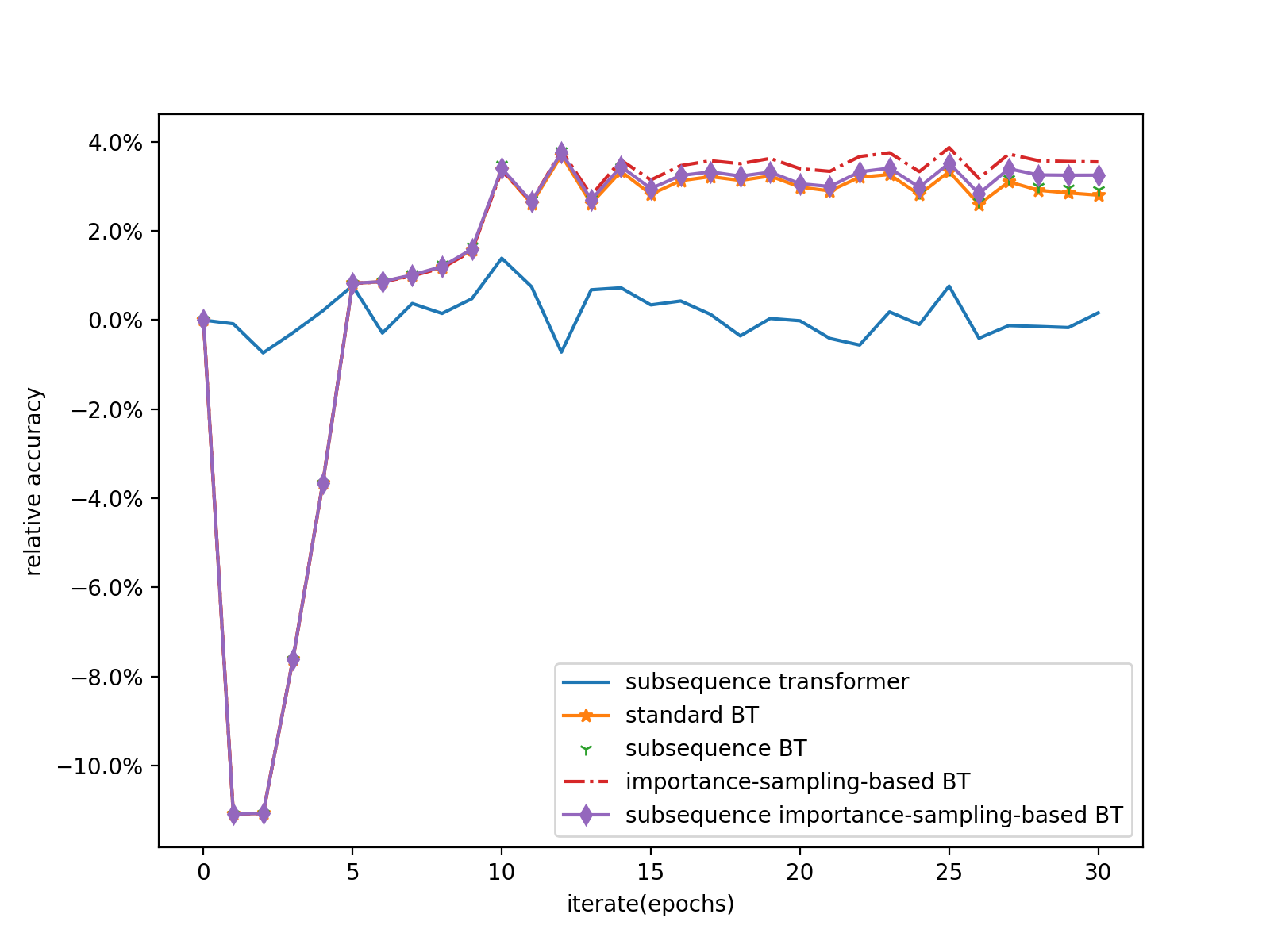}
  \captionof{figure}{Relative Accuracy on IMDB}
  \label{fig:imdb}
\end{minipage}%
\begin{minipage}[t]{0.5\textwidth}
  \centering
  \includegraphics[width=\linewidth]{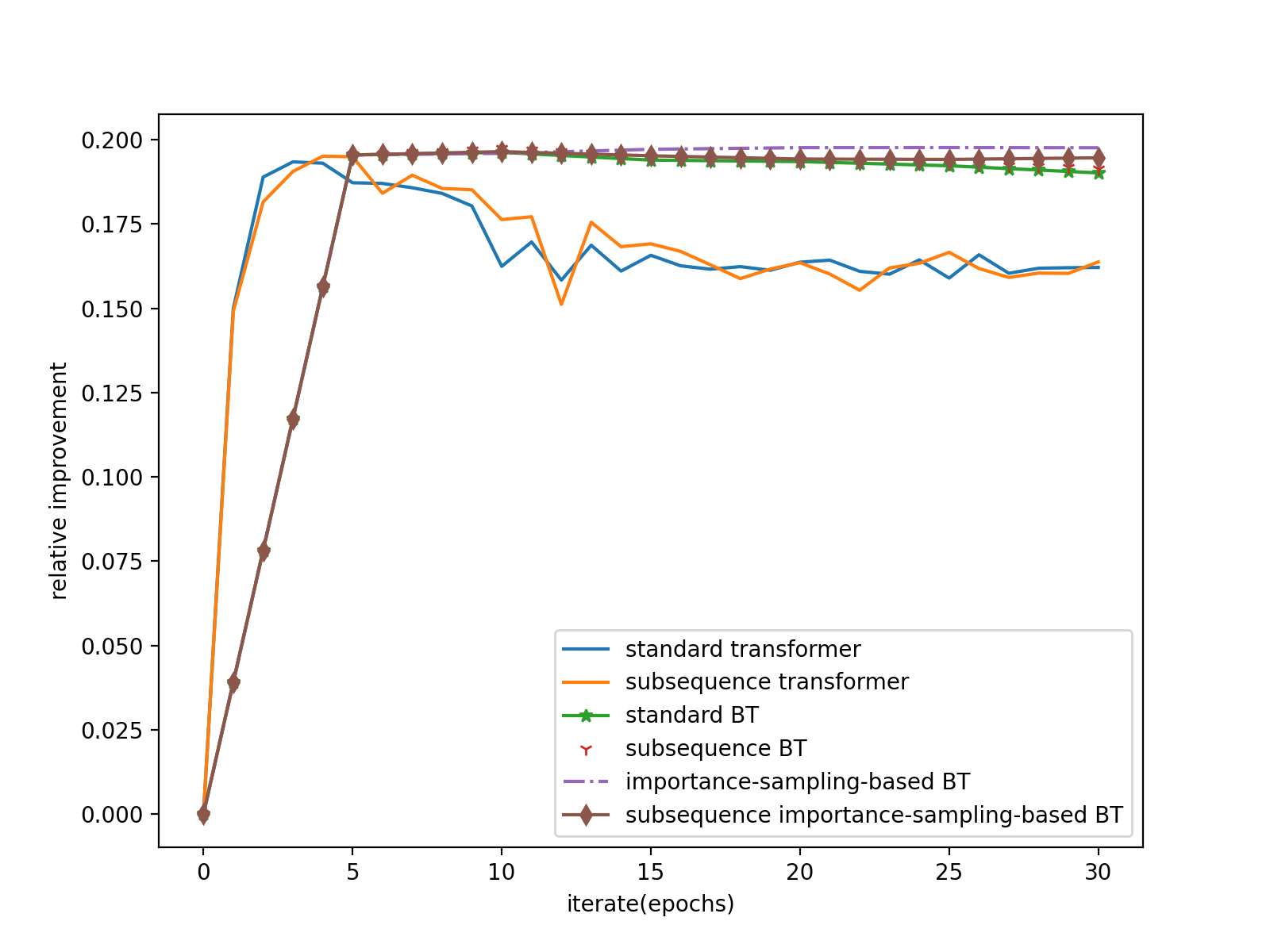}
  \captionof{figure}{Improvement on IMDB}
  \label{fig:imdb_imp}
\end{minipage}
\begin{minipage}[t]{0.5\textwidth}
  \centering
  \includegraphics[width=\linewidth]{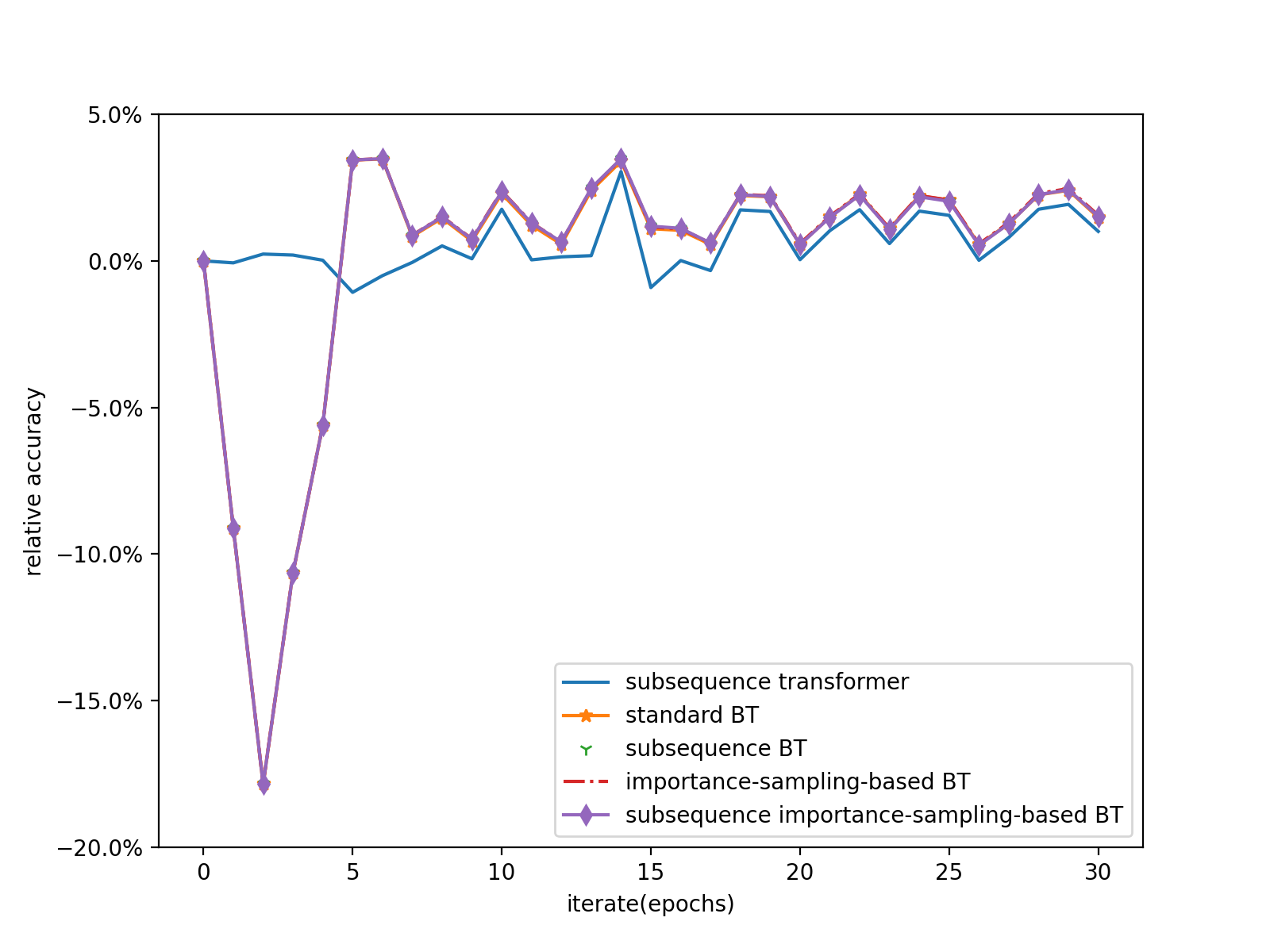}
  \captionof{figure}{Relative Accuracy on Yelp}
  \label{fig:yelp}
\end{minipage}%
\begin{minipage}[t]{0.5\textwidth}
  \centering
  \includegraphics[width=\linewidth]{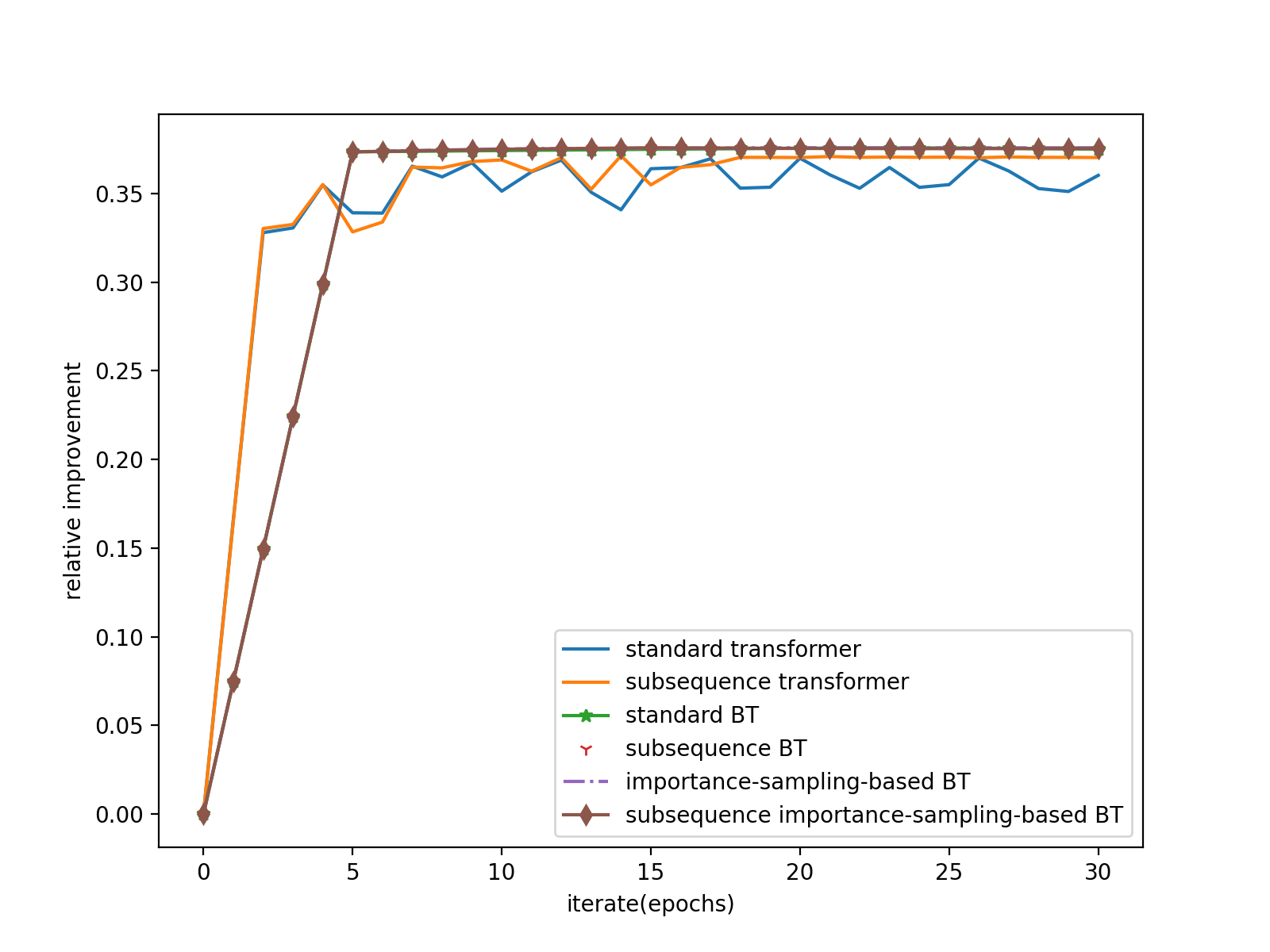}
  \captionof{figure}{Improvement on Yelp}
  \label{fig:yelp_imp}
\end{minipage}
\begin{minipage}[t]{.5\textwidth}
  \centering
  \includegraphics[width=\linewidth]{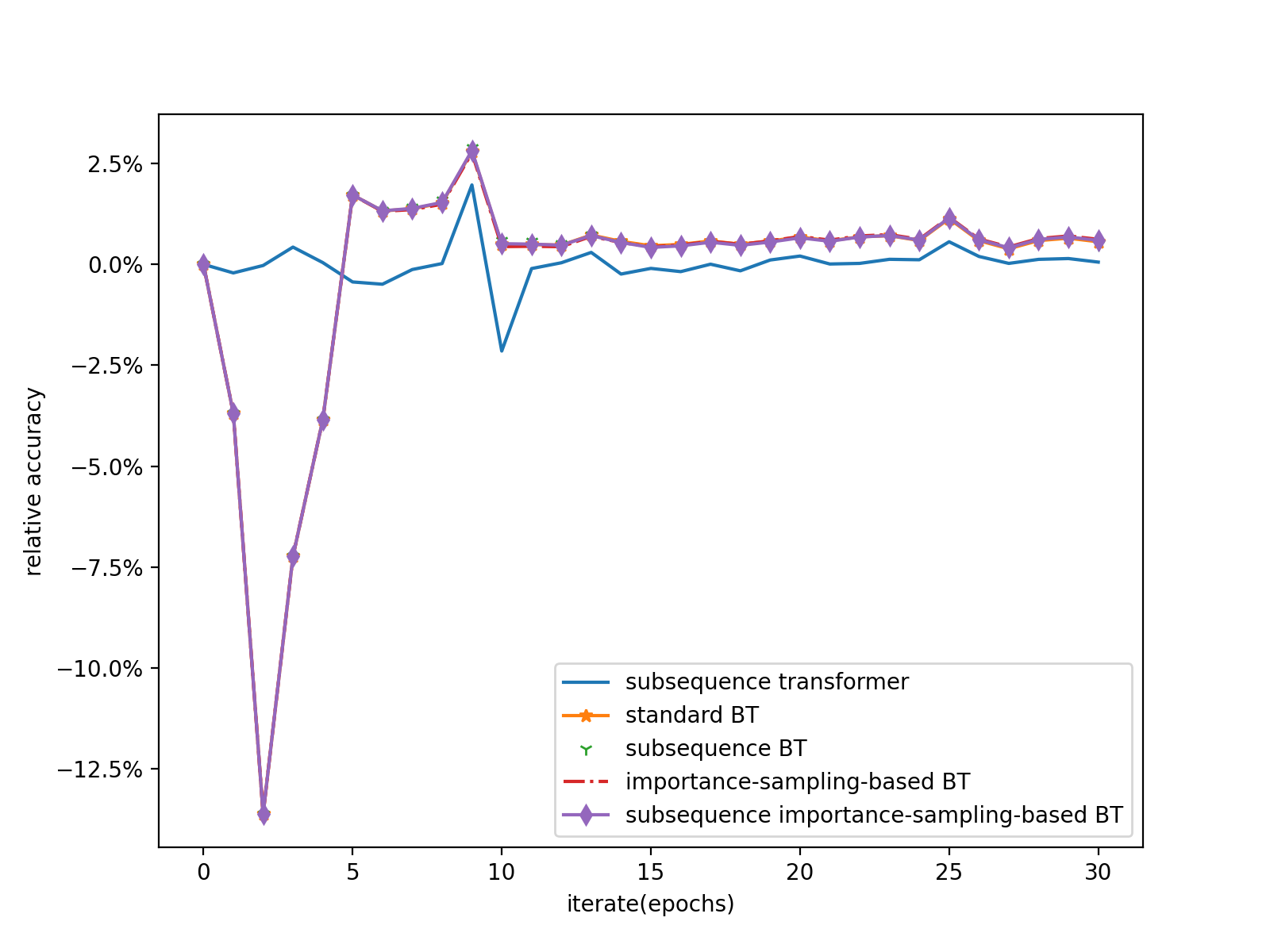}
  \captionof{figure}{Relative Accuracy on Amazon}
  \label{fig:amazon}
\end{minipage}%
\begin{minipage}[t]{.5\textwidth}
  \centering
  \includegraphics[width=\linewidth]{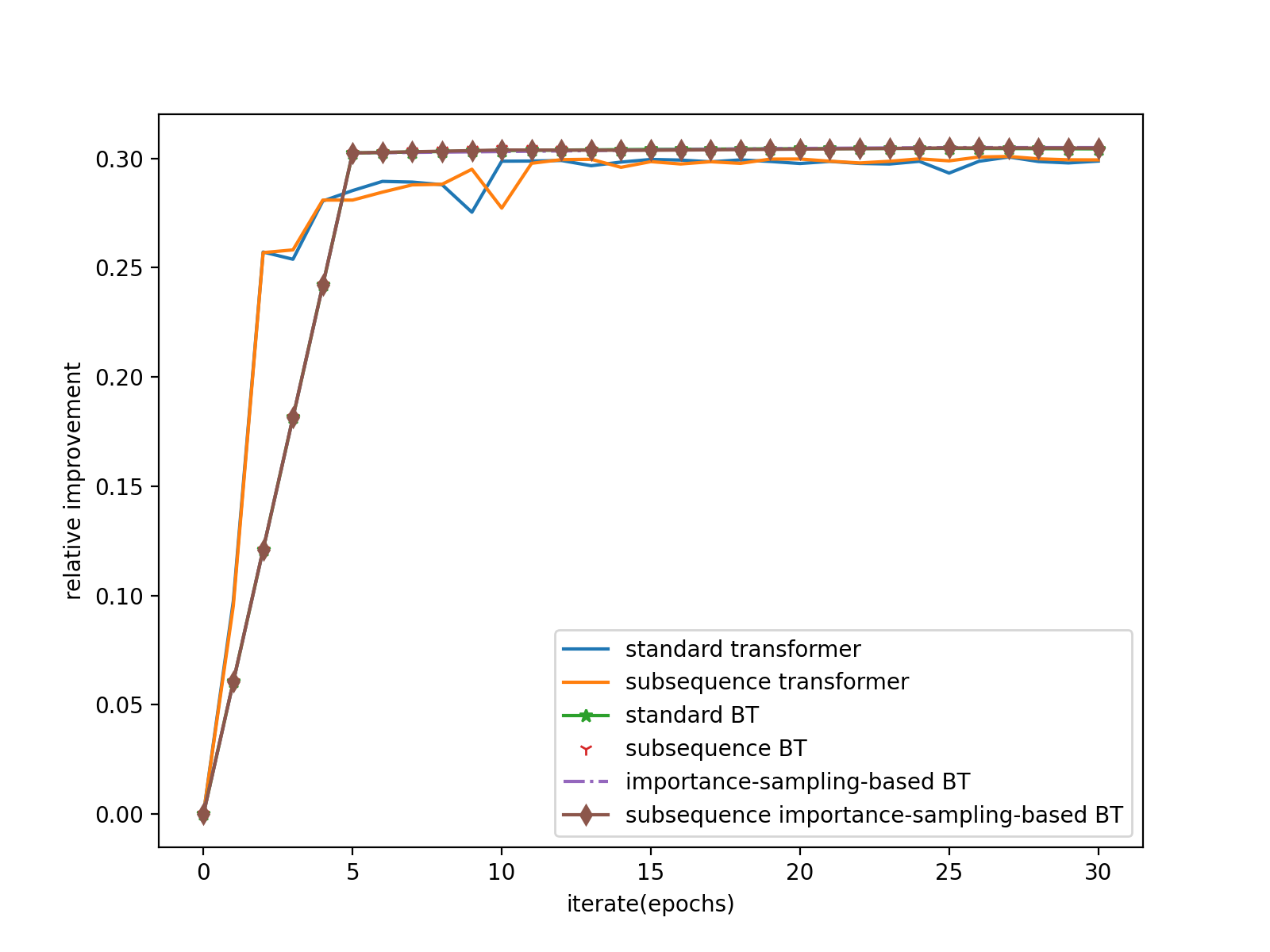}
  \captionof{figure}{Improvement on Amazon}
  \label{fig:amazon_imp}
\end{minipage}
\end{figure*}

The weak learner used is RoBERTa-based \cite{liu2019roberta} from the HuggingFace library with only word embeddings to be pre-trained weights. Using transformer based boosting algorithms, we train an ensemble of 6 transformers each with 5 epochs (these numbers yield good performance). In subsequence BoostTransformer, we pick the most important $80\%$ of the tokens in the vocabulary and reconstruct the dataset based on this new vocabulary. In importance-sampling-based BoostTransformer, the first flavor, in each iteration, we select $80\%$ of the samples based on the probability distribution in (\ref{importance sample distribution}) without further subsequence technique. In subsequence importance-sampling-based BoostTransformer, we first select $80\%$ of the samples based on the probability distribution in (\ref{importance sample distribution}), and then pick the most important $80\%$ of the tokens in the current vocabulary given by the selected $80\%$ samples, after that, we reconstruct the dataset based on this modified vocabulary. For comparison, we train the vanilla transformer and subsequence transformer, which randomly removes $20\%$ of the tokens and trains the network on the dataset for $30$ epochs. To train the model, we use AdamW \cite{Loshchilov2019DecoupledWD} with learning rate $10^{-5}$, weight decay $0.01$ and batch size $16$. We use linear learning rate decay with warmup ratio 0.06. 

We start by presenting the three public datasets used: IMDB \cite{Maas2011LearningWV}, Yelp polarity reviews and Amazon polarity reviews \cite{McAuley2013HiddenFA}. The IMDB dataset, which is for binary sentiment classification, contains a set of 25,000 highly polar movie reviews for training, and 25,000 for testing. The Yelp polarity reviews dataset, which is a subset of the dataset obtained from the Yelp Dataset Challenge in 2015, consists of $100,000$ training samples and $38,000$ testing samples. The classification task for this dataset is predicting a polarity label by considering stars 1 and 2 negative, and 3 and 4 positive for each review text. The last dataset we use is the Amazon polarity reviews dataset, which is a subset of the original Amazon reviews dataset from the Stanford Network Analysis Project (SNAP). Dealing with the same classification task as the Yelp polarity review dataset, the Amazon polarity reviews dataset contains $100,000$ training samples and $25,000$ testing samples. The subsampled datasets are standard, i.e. we did not create our own subsamples. Empirically we found that a weak learner with 6 heads and 6 layers achieves good robust performance.

Given the architecture of the weak learner, we start by discussing experiments on IMDB. In Figure \ref{fig:imdb}, we compare the relative performances of the algorithms with respect to the vanilla transformer. As it shows, all versions of BoostTransformer do not perform as good as the standard transformer and subsequence transformer in the first few epochs. However, they catch up quickly and dominate the performance in the remaining training epochs. Even more, based on Figure \ref{fig:imdb_imp}, which represents each model's relative improvement with respect to its initial weights, all versions of BoostTransformer maintain their performances as the number of epochs increases, while the performances of the standard transformer and the subsequence transformer start decreasing and fluctuating dramatically after the first few epochs, which implies that all versions of BoostTransformer are more robust than the standard and subsequence transformer. Next, we evaluate relative performances with respect to the vanilla transformer on the Yelp and Amazon polarity review datasets. From Figures \ref{fig:yelp}-\ref{fig:amazon_imp}, we discover that the superior and more robust behavior of boosting algorithms over transformer is vigorous.

 \begin{figure}[!t]
  \centering
  \begin{minipage}[t]{0.5\textwidth}
    \centering
    \includegraphics[width=\linewidth]{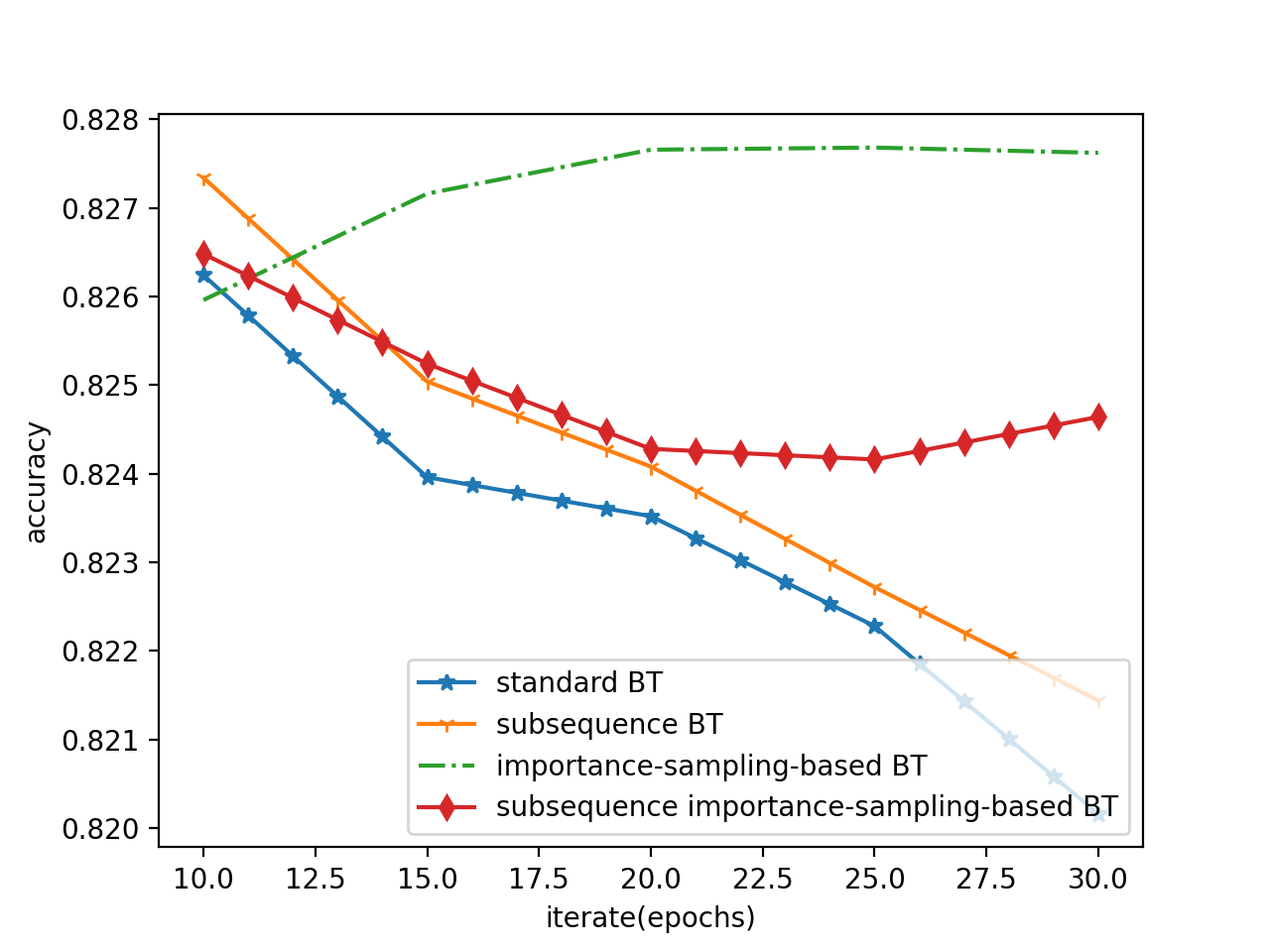}
    \caption{Zoomed-in attention heatmap for IMDB.}
    \label{fig:imdb_zoom}
  \end{minipage}
  \hfill
  \begin{minipage}[t]{0.5\textwidth}
    \centering
    \includegraphics[width=\linewidth]{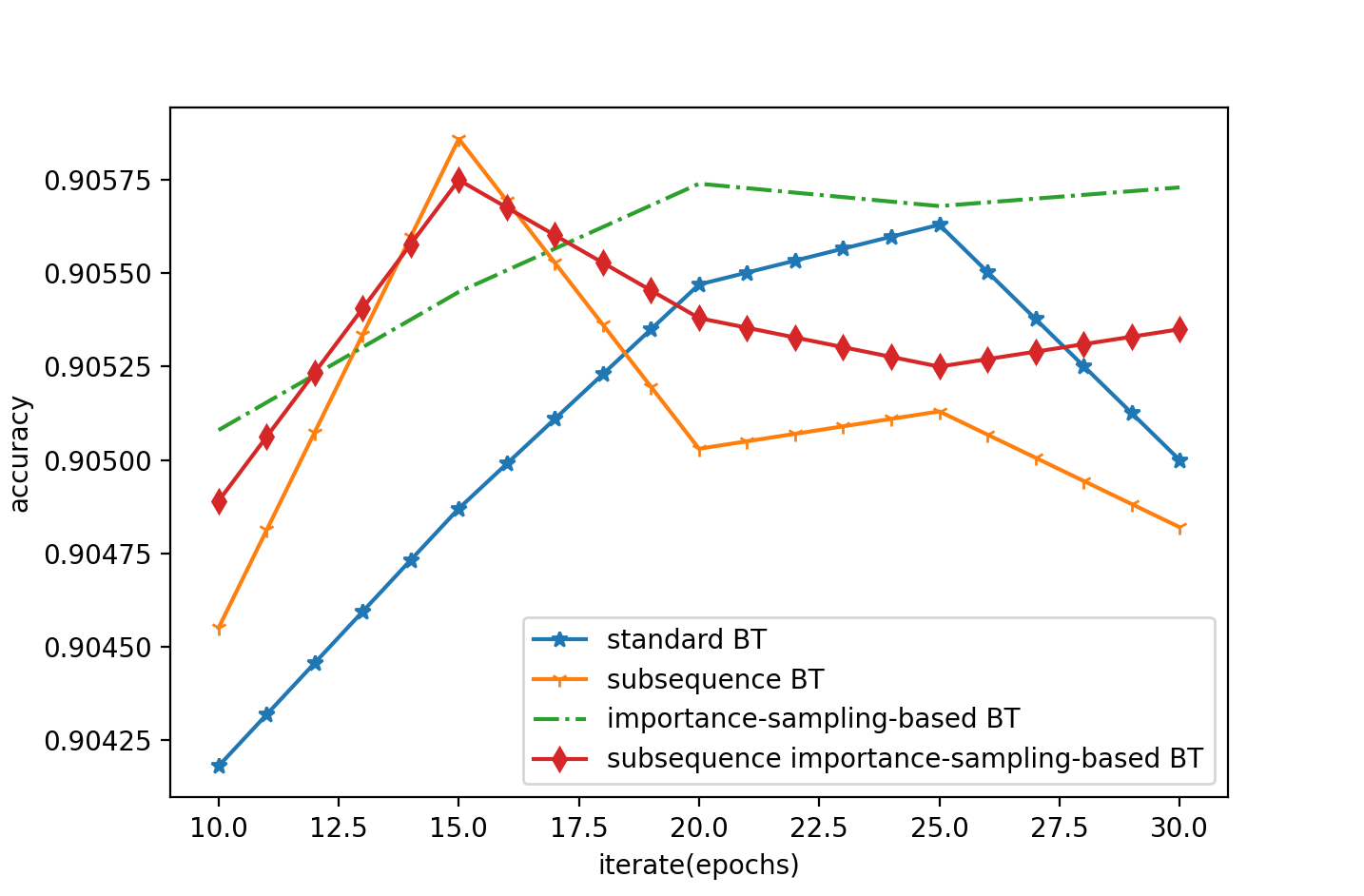}
    \caption{Zoomed-in attention heatmap for Yelp.}
    \label{fig:yelp_zoom}
  \end{minipage}
  \hfill
  \begin{minipage}[t]{0.5\textwidth}
    \centering
    \includegraphics[width=\textwidth]{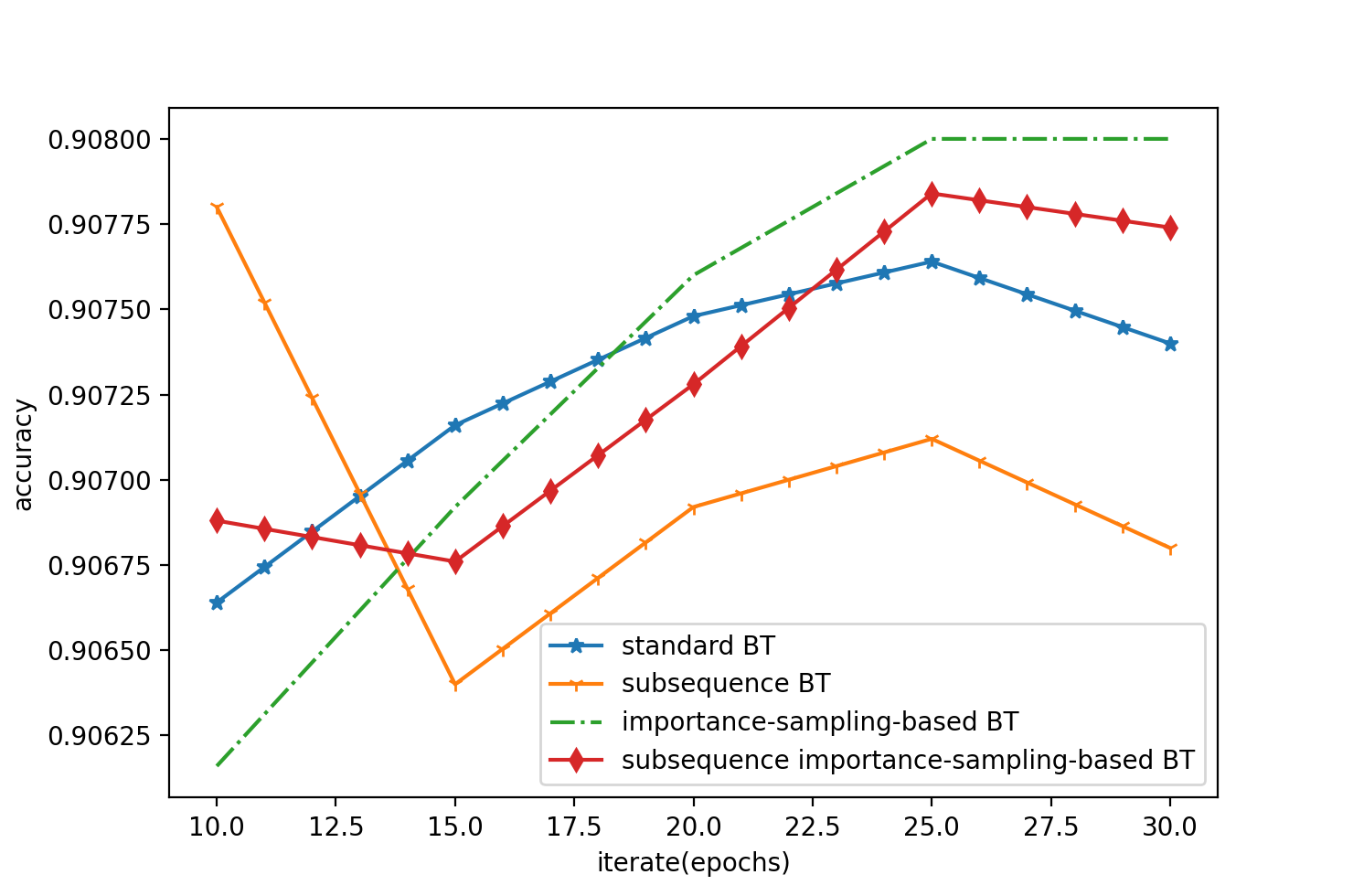}
    \caption{Zoomed-in attention heatmap for Amazon.}
    \label{fig:amazon_zoom}
    \end{minipage}
\end{figure}

\begin{table*}[t]
\scriptsize
\centering
\resizebox{\linewidth}{!}{%
\begin{tabular}{|l|r|r|r|r|r|r|}
\hline
& \multicolumn{1}{c|}{Transformer} 
& \multicolumn{1}{c|}{\begin{tabular}[c]{@{}c@{}}Subsequence\\Transformer\end{tabular}} 
& \multicolumn{1}{c|}{BoostTransformer} 
& \multicolumn{1}{c|}{\begin{tabular}[c]{@{}c@{}}Subsequence\\BoostTransformer\end{tabular}} 
& \multicolumn{1}{c|}{\begin{tabular}[c]{@{}c@{}}Importance-Sampling\\BoostTransformer\end{tabular}} 
& \multicolumn{1}{c|}{\begin{tabular}[c]{@{}c@{}}Subsequence\\Importance-Sampling\\BoostTransformer\end{tabular}} \\ \hline
IMDB   & 21 & 14 & 23 & 17 & 13 & 12 \\ \hline
Yelp   & 76 & 52 & 84 & 66 & 52 & 46 \\ \hline
Amazon & 75 & 53 & 79 & 58 & 46 & 41 \\ \hline
\end{tabular}%
}
\caption{
Running time (in seconds) for Transformer variants across IMDB, Yelp, and Amazon. "Subsequence" uses attention-guided token selection; "Importance-Sampling" selects samples by residuals.
}
\label{tb:timeTrans}
\end{table*}

Furthermore, we zoom in on the performances at iterates larger than 2. In Figures \ref{fig:imdb_zoom}-\ref{fig:amazon_zoom}, compared with the standard BoostTransformer, we observe that the subsequence BoostTransformer, importance-sampling-based BoostTransformer and subsequence importance-sampling-based BoostTransformer demonstrate a superior performance. Therefore, we conclude that the subsequence and importance sampling techniques are beneficial for the boosting algorithms. Moreover, we observe that the importance-sampling-based BoostTransformer gradually improves its performance and maintains its performance later on, while the subsequence BoostTransformer hits its best accuracy in early epochs and then starts fluctuating and decaying.  The gap between the importance-sampling-based BoostTransformer and the subsequence BoostTransformer is more significant on the IMDB dataset, which has a much smaller size than the Yelp and Amazon polarity review datasets.

For the subsequence importance-sampling-based BoostTransformer, compared to the subsequence BoostTransformer, although the subsequence importance-sampling-based BoostTransformer does not fluctuate and decrease as much as the subsequence BoostTransformer, which is more obvious in a small dataset (i.e. the IMDB dataset), its best accuracy is lower than that of the subsequence BoostTransformer, which is more obvious in larger datasets (i.e. the Yelp and Amazon datasets). On the other hand, compared to the importance-sampling-based BoostTransformer, although the subsequence importance-sampling-based BoostTransformer obtains its best accuracy earlier than the importance-sampling-based BoostTransformer,  its overall performance fluctuates while the importance-sampling-based BoostTransformer keeps increasing and maintains its high-quality performance in all of the datasets, which implies that the subsequence importance-sampling-based BoostTransformer is less stable than the importance-sampling-based BoostTransformer. In conclusion, the subsequence BoostTransformer fits well for datasets with enough samples and the importance-sampling-based BoostTransformer is more suitable for datasets with a limited number of samples. As we can see from Table \ref{tb:timeTrans}, the subsequence technique not only improves the performance of the boosting algorithms but also reduces the running time. Furthermore, the importance sampling technique reduces the running time significantly without hurting the performance.

\section{Limitation}
When minimizing runtime is the top priority, the Subsequence Transformer offers the fastest training with acceptable performance. However, if both accuracy and efficiency are critical, especially for large datasets like Yelp and Amazon—the Subsequence BoostTransformer provides a compelling balance, achieving superior and more stable performance with only a slight increase in runtime. Although the Subsequence Importance-Sampling BoostTransformer combines both strategies, it shows less stability and a lower peak accuracy on larger datasets, suggesting that its benefits are context-dependent. Overall, the Subsequence BoostTransformer is best suited for large, resource-constrained tasks, while the Importance-Sampling BoostTransformer is preferable for small data regimes where robustness is essential.

\bibliographystyle{IEEEtran} 
\bibliography{custom}     

\begin{thebibliography}{10}
\providecommand{\url}[1]{#1}
\csname url@samestyle\endcsname
\providecommand{\newblock}{\relax}
\providecommand{\bibinfo}[2]{#2}
\providecommand{\BIBentrySTDinterwordspacing}{\spaceskip=0pt\relax}
\providecommand{\BIBentryALTinterwordstretchfactor}{4}
\providecommand{\BIBentryALTinterwordspacing}{\spaceskip=\fontdimen2\font plus
\BIBentryALTinterwordstretchfactor\fontdimen3\font minus \fontdimen4\font\relax}
\providecommand{\BIBforeignlanguage}[2]{{%
\expandafter\ifx\csname l@#1\endcsname\relax
\typeout{** WARNING: IEEEtran.bst: No hyphenation pattern has been}%
\typeout{** loaded for the language `#1'. Using the pattern for}%
\typeout{** the default language instead.}%
\else
\language=\csname l@#1\endcsname
\fi
#2}}
\providecommand{\BIBdecl}{\relax}
\BIBdecl

\bibitem{batra2021bert_sentiment}
H.~Batra, N.~S. Punn, and S.~Agarwal, ``Bert-based sentiment analysis: A software engineering perspective,'' \emph{arXiv preprint arXiv:2106.02581}, 2021.

\bibitem{munikar2019finegrained}
M.~Munikar, S.~Shakya, and A.~Shrestha, ``Fine-grained sentiment classification using bert,'' in \emph{Proceedings of the International Conference on Deep Learning Technologies}, 2019.

\bibitem{raffel2020exploring}
C.~Raffel and et~al., ``Exploring the limits of transfer learning with a unified text-to-text transformer,'' \emph{JMLR}, 2020.

\bibitem{brown2020language}
T.~Brown and et~al., ``Language models are few-shot learners,'' \emph{NeurIPS}, 2020.

\bibitem{Karita2019ImprovingTE}
T.~Nakatani, ``Improving transformer-based end-to-end speech recognition with connectionist temporal classification and language model integration,'' in \emph{INTERSPEECH}, 2019.

\bibitem{Yeh2019TransformerTransducerES}
C.-F. Yeh, J.~Mahadeokar, K.~Kalgaonkar, Y.~Wang, D.~Le, M.~Jain, K.~Schubert, C.~Fuegen, and M.~L. Seltzer, ``Transformer-{T}ransducer: End-to-end speech recognition with self-attention,'' \emph{ArXiv}, vol. abs/1910.12977, 2019.

\bibitem{Dong2018SpeechTransformerAN}
L.~Dong, S.~Xu, and B.~Xu, ``Speech-{T}ransformer: A no-recurrence sequence-to-sequence model for speech recognition,'' in \emph{ICASSP}, 2018.

\bibitem{Li2019NeuralSS}
N.~Li, S.~Liu, Y.~Liu, S.~Zhao, and M.~Liu, ``Neural speech synthesis with transformer network,'' in \emph{AAAI}, 2019.

\bibitem{Gangi2019AdaptingTT}
M.~A. Di~Gangi, M.~Negri, and M.~Turchi, ``Adapting transformer to end-to-end spoken language translation,'' in \emph{INTERSPEECH}, 2019.

\bibitem{Zhou2018SyllableBasedSS}
S.~Zhou, L.~Dong, S.~Xu, and B.~Xu, ``Syllable-based sequence-to-sequence speech recognition with the transformer in mandarin chinese,'' \emph{ArXiv}, vol. abs/1804.10752, 2018.

\bibitem{chen2021autoformer}
X.~Chen, C.~Zhang, and C.-J. Hsieh, ``Autoformer: Searching transformers for visual recognition,'' \emph{arXiv preprint arXiv:2106.13008}, 2021.

\bibitem{so2021primer}
D.~R. So, Q.~V. Le, and C.~Liang, ``Primer: Searching for efficient transformers,'' in \emph{International Conference on Learning Representations (ICLR)}, 2022.

\bibitem{wang2022adamix}
Y.~Wang, W.~Wang, X.~Lin, Z.~Lin, and Y.~Liu, ``Adamix: Mixture-of-adapters for parameter-efficient fine-tuning,'' \emph{arXiv preprint arXiv:2110.04244}, 2022.

\bibitem{Vaswani2017AttentionIA}
A.~Vaswani, N.~Shazeer, N.~Parmar, J.~Uszkoreit, L.~Jones, A.~Gomez, L.~Kaiser, and I.~Polosukhin, ``Attention is all you need,'' in \emph{NIPS}, 2017.

\bibitem{Bahdanau2015NeuralMT}
D.~Bahdanau, K.~Cho, and Y.~Bengio, ``Neural machine translation by jointly learning to align and translate,'' \emph{CoRR}, vol. abs/1409.0473, 2015.

\bibitem{Sutskever2014SequenceTS}
I.~Sutskever, O.~Vinyals, and Q.~V. Le, ``Sequence to sequence learning with neural networks,'' \emph{ArXiv}, vol. abs/1409.3215, 2014.

\bibitem{Chen2018HybridGB}
H.~Chen, S.~Lundberg, and S.-I. Lee, ``Hybrid gradient boosting trees and neural networks for forecasting operating room data,'' \emph{ArXiv}, vol. abs/1801.07384, 2018.

\bibitem{Assaad2008ANB}
M.~Assaad, R.~Bon{\'e}, and H.~Cardot, ``A new boosting algorithm for improved time-series forecasting with recurrent neural networks,'' \emph{Information Fusion}, vol.~9, pp. 41--55, 2008.

\bibitem{Clark2019WhatDB}
K.~Clark, U.~Khandelwal, O.~Levy, and C.~D. Manning, ``What does {BERT} look at? {A}n analysis of {BERT}'s attention,'' \emph{ArXiv}, vol. abs/1906.04341, 2019.

\bibitem{Needell2014StochasticGD}
D.~Needell, R.~Ward, and N.~Srebro, ``Stochastic gradient descent, weighted sampling, and the randomized {K}aczmarz algorithm,'' \emph{Mathematical Programming}, vol. 155, pp. 549--573, 2014.

\bibitem{Zhao2015StochasticOW}
P.~Zhao and T.~Zhang, ``Stochastic optimization with importance sampling for regularized loss minimization,'' in \emph{ICML}, 2015.

\bibitem{Katharopoulos2018NotAS}
A.~Katharopoulos and F.~Fleuret, ``Not all samples are created equal: Deep learning with importance sampling,'' \emph{ArXiv}, vol. abs/1803.00942, 2018.

\bibitem{Csiba2018ImportanceSF}
D.~Csiba and P.~Richt{\'a}rik, ``Importance sampling for minibatches,'' \emph{ArXiv}, vol. abs/1602.02283, 2018.

\bibitem{Devlin2019BERTPO}
J.~Devlin, M.-W. Chang, K.~Lee, and K.~Toutanova, ``{BERT}: Pre-training of deep bidirectional transformers for language understanding,'' in \emph{NAACL-HLT}, 2019.

\bibitem{liu2019roberta}
Y.~Liu, M.~Ott, N.~Goyal, J.~Du, M.~Joshi, D.~Chen, O.~Levy, M.~Lewis, L.~Zettlemoyer, and V.~Stoyanov, ``Roberta: A robustly optimized bert pretraining approach,'' \emph{arXiv preprint arXiv:1907.11692}, 2019.

\bibitem{Brahimi2019BoostedCN}
S.~Brahimi, N.~B. Aoun, and C.~B. Amar, ``Boosted convolutional neural network for object recognition at large scale,'' \emph{Neurocomputing}, vol. 330, pp. 337--354, 2019.

\bibitem{Alain2015VarianceRI}
G.~Alain, A.~Lamb, C.~Sankar, A.~C. Courville, and Y.~Bengio, ``Variance reduction in {SGD} by distributed importance sampling,'' \emph{ArXiv}, vol. abs/1511.06481, 2015.

\bibitem{Loshchilov2019DecoupledWD}
I.~Loshchilov and F.~Hutter, ``Decoupled weight decay regularization,'' in \emph{ICLR}, 2019.

\bibitem{Maas2011LearningWV}
A.~Maas, R.~E. Daly, P.~T. Pham, D.~Huang, A.~Y. Ng, and C.~Potts, ``Learning word vectors for sentiment analysis,'' in \emph{ACL}, 2011.

\bibitem{McAuley2013HiddenFA}
J.~McAuley and J.~Leskovec, ``Hidden factors and hidden topics: understanding rating dimensions with review text,'' in \emph{RecSys}, 2013.

\bibitem{Hastie2009MulticlassA}
T.~Hastie, S.~Rosset, J.~Zhu, and H.~Zou, ``Multi-class {A}da{B}oost,'' \emph{Statistics and Its Interface}, vol.~2, pp. 349--360, 2009.

\bibitem{Mukherjee2013ATO}
I.~Mukherjee and R.~E. Schapire, ``A theory of multiclass boosting,'' \emph{Journal of Machine Learning Research}, vol.~14, pp. 437--497, 2013.

\bibitem{Saberian2011MulticlassBT}
M.~J. Saberian and N.~Vasconcelos, ``Multiclass {B}oosting: Theory and algorithms,'' in \emph{NIPS}, 2011.

\bibitem{Moghimi2016BoostedCN}
M.~Moghimi, S.~J. Belongie, M.~J. Saberian, J.~Yang, N.~Vasconcelos, and L.-J. Li, ``Boosted convolutional neural networks,'' in \emph{BMVC}, 2016.

\end{thebibliography}

\newpage
\onecolumn
\section*{Appendix}
\addcontentsline{toc}{section}{Appendix}  

\section*{A. Background (BoostCNN)}
\label{sec:background}
We start with a brief overview of multiclass boosting. Given a sample $x_i\in\mathcal{X}$ and its class label $z_i\in\left\{1,2,\cdots,M\right\}$, multiclass boosting is a method that combines several multiclass predictors $g_t:\mathcal{X}\rightarrow \mathbb{R}^d$ to form a strong committee $f(x)$ of classifiers, i.e. $f(x)=\sum_{t=1}^N\alpha_t g_t(x)$ where $g_t$ and $\alpha_t$ are the weak learner and coefficient selected at the $t^{\mathrm{th}}$ boosting iteration. There are various approaches for multiclass boosting such as \cite{Hastie2009MulticlassA}, \cite{Mukherjee2013ATO}, \cite{Saberian2011MulticlassBT}; we use the GD-MCBoost method of \cite{Saberian2011MulticlassBT}, \cite{Moghimi2016BoostedCN} herein. For simplicity, in the rest of the paper, we assume that $d=M$.

Standard BoostCNN \cite{Moghimi2016BoostedCN} trains a boosted predictor $f(x)$ by minimizing the risk of classification
\begin{align}
\label{risk function}
    \mathcal{R}[f]&=\mathrm{E}_{X,Z}\left[L(z,f(x)) \right]\nonumber\\
    &\approx \frac{1}{\left|\mathcal{D} \right|}\sum_{(x_i,z_i)\in\mathcal{D}}L(z_i,f(x_i)),
\end{align}
where $\mathcal{D}$ is the set of training samples and 
\begin{align*}
    L(z,f(x))=\sum_{j=1, j\neq z}^M e^{-\frac{1}{2}\left[ \left\langle y_{z},f(x)\right\rangle -\left\langle y_j,f(x) \right\rangle \right]},
\end{align*}
given $y_k=\mathds{1}_k\in\mathbb{R}^M$, i.e. the $k^{\mathrm{th}}$ unit vector. The minimization is via gradient descent in a functional space. Standard BoostCNN starts with $f(x)=\mathbf{0}\in\mathbb{R}^d$ for every $x$ and iteratively computes the directional derivative of risk (\ref{risk function}), for updating $f(x)$ along the direction of $g(x)$
\begin{align}
    \delta \mathcal{R}[f;g]&=\left.\frac{\partial \mathcal{R}[f+\epsilon g]}{\partial \epsilon}\right|_{\epsilon=0}\nonumber\\
    &=-\frac{1}{2\left| \mathcal{D}\right|}\sum_{(x_i,z_i)\in\mathcal{D}}\sum_{j=1}^Mg_j(x_i)w_j(x_i,z_i)\nonumber\\
    &=-\frac{1}{2\left| \mathcal{D}\right|}\sum_{(x_i,z_i)\in\mathcal{D}} g(x_i)^Tw(x_i,z_i),
    \label{functional gradient}
\end{align}
where 
\begin{align}
\label{weight update}
    w_k(x,z)=\left\{\begin{matrix}
&-e^{-\frac{1}{2}\left [ f_{z}(x)-f_k(x) \right ]},\quad k\neq z\\ 
&\sum_{j=1,j\neq k}^M e^{-\frac{1}{2}\left [ f_{z}(x)-f_j(x) \right ]}, \quad k=z,
\end{matrix}\right.
\end{align}
and $g_j(x_i)$ computes the directional derivative along $\mathds{1}_j$. Then, standard BoostCNN selects a weak learner $g^*$ that minimizes (\ref{functional gradient}), which essentially measures the similarity between the boosting weights $w(x_i,z_i)$ and the function values $g(x_i)$. Therefore, the optimal network output $g^*(x_i)$ has to be proportional to the boosting weights, i.e. 
\begin{align}
\label{eq:grad-weight}
    g^*(x_i)=\beta w(x_i, z_i),
\end{align}
for some constant $\beta > 0$. Note that the exact value of $\beta$ is irrelevant since $g^*(x_i)$ is scaled when computing $\alpha^*$. Consequently, without loss of generality, we assume $\beta=1$ and convert the problem to finding a network $g(x)\in \mathbb{R}^M$ that minimizes the square error loss
\begin{align}
\label{weak learner train}
    \mathcal{L}(w,g)=\sum_{(x_i,z_i)\in\mathcal{D}}\left\|g(x_i)-w(x_i,z_i)\right\|^2.
\end{align}
After the weak learner is trained, BoostCNN applies a line search to compute the optimal step size along $g^*$,

\begin{align}
    \alpha^*=\argmin_{\alpha\in\mathbb{R}}\mathcal{R}[f+\alpha g^*].
    \label{boost parameter}
\end{align}

Finally, the boosted predictor $f(x)$ is updated as $f=f+\alpha^* g^*$.

\section*{B. Proof of Theorem \hyperref[thm:1]{1}\label{pf:thm1}}
\begin{proof}
Given a probability distribution $P$ for dataset $(x_i,z_i)$, by assumption, the stochastic gradient of the loss function is unbiased, i.e.
\begin{align}
\label{eq:unbiased grad}
\mathbb{E}_{P}\left[\frac{\partial \bar{L}_i(z_i,f(x_i))}{\partial f} \right]=\frac{\partial \mathcal{R}}{\partial f},
\end{align}
with
\begin{align*}
  \mathcal{R}[f] = \frac{1}{\left|\mathcal{D}\right|}\sum_{(x_i,z_i)\in\mathcal{D}}L(z_i,f(x_i))=\mathbb{E}_{P}\left[\bar{L}_i(z_i,f(x_i)) \right]
\end{align*}
and
\begin{align*}
    \bar{L}_i(z_i,f(x_i))=\frac{1}{\left|\mathcal{D}\right|P(I=i)}L(z_i,f(x_i)).
\end{align*}

At iterate $t$, in importance-sampling-based Boosting algorithms, given probability distribution $P_t$ and $P_t^i = P_t(I=i)$, the current gradient given a subset $\mathcal{I}^t$ of samples is
\begin{align}
\label{eq:grad_explain}
    \bar{g}_t^{\mathcal{I}^t}&=\frac{1}{\left|\mathcal{I}^t\right|}\sum_{(x_i,z_i)\in\mathcal{I}^t}\left.\frac{\partial \bar{L}_i(z_i,f_{t-1}(x_i)+\epsilon g(x_i))}{\partial g}\right|_{\epsilon =0}\nonumber\\
    &=\frac{1}{\left|\mathcal{I}^t\right|}\sum_{(x_i,z_i)\in\mathcal{I}^t}\frac{1}{\left|\mathcal{D}\right|P^i_t}\left.\frac{\partial L_i(z_i,f_{t-1}(x_i)+\epsilon g(x_i))}{\partial g}\right|_{\epsilon =0}\nonumber\\
    &= \frac{1}{\left|\mathcal{I}^t\right|}\sum_{(x_i,z_i)\in\mathcal{I}^t}\frac{1}{\left|\mathcal{D}\right|P^i_t} g_t^i=\frac{1}{\left|\mathcal{I}^t\right|}\sum_{k=1}^{\left|\mathcal{I}^t\right|}G_k,
\end{align}
where $g_t^i = \left.\frac{\partial L_i(z_i,f_{t-1}(x_i)+\epsilon g(x_i))}{\partial g}\right|_{\epsilon =0}$ and $G_k$ is the random variable corresponding to sample $k$. Note that 
\begin{align}
    g_t=\frac{\partial \mathcal{R}[f_{t-1};g]}{\partial g}=\left.\frac{\partial \mathcal{R}[f_{t-1}+\epsilon g]}{\partial g}\right|_{\epsilon =0}
\end{align}
and
\begin{align}
\label{eq:importance grad}
    \mathbb{E}_{P_t}(G_t)=\mathbb{E}_{P_t}\left[\bar{g}_t^{\mathcal{I}^t} \right]=g_t,
\end{align}
due to the unbiased gradient in (\ref{eq:unbiased grad}). Given $\bar{g}_{t}^{\mathcal{I}^{t}}$ computed on a subset $\mathcal{I}^{t}$ with probability distribution $P_{t}$, we consider
\begin{align}
    \mathbb{E}_{P_{t}}\left[ \Delta^{(t)} \right]&=\left\|f_{t-1}-f^*\right\|^2-\mathbb{E}_{P_{t}}\left[\left.\left\| f_{t}-f^*\right\|^2 \right|\mathcal{F}^{t-1}\right]\nonumber\\
    &=\left\|f_{t-1}-f^*\right\|^2-\mathbb{E}_{P_{t}}\left[\left.\left\| f_{t-1}+\alpha_{t}\bar{g}_{t}^{\mathcal{I}^{t}}-f^*\right\|^2 \right|\mathcal{F}^{t-1}\right]\nonumber\\
    &=-2\alpha_{t}\left\langle f_{t-1}-f^*, \mathbb{E}_{P_{t}}\left[\left.\bar{g}_{t}^{\mathcal{I}^{t}}\right|\mathcal{F}^{t-1}\right]\right\rangle -\alpha_{t}^2 \mathbb{E}_{P_{t}}\left[\left.\left\| \bar{g}_{t}^{\mathcal{I}^{t}}\right\|^2 \right|\mathcal{F}^{t-1}\right].\label{eq:diff1}
\end{align}
By inserting (\ref{eq:importance grad}) into (\ref{eq:diff1}), we have
\begin{align}
\label{eq:diff2}
     \mathbb{E}_{P_{t}}\left[ \Delta^{(t)} \right]=-2\alpha_{t}\left\langle f_{t-1}-f^*, g_{t}\right\rangle -\alpha_{t}^2 \mathbb{E}_{P_{t}}\left[\left.\left\| \bar{g}_{t}^{\mathcal{I}^{t}}\right\|^2 \right|\mathcal{F}^{t-1}\right].
\end{align}
Thus, maximizing $\mathbb{E}_{P_{t}}\left[ \Delta^{(t)} \right]$ is equivalent to minimizing the variance of the gradient, i.e. $\mathbb{E}_{P_{t}}\left[\left.\left\| \bar{g}_{t}^{\mathcal{I}^{t}}\right\|^2 \right|\mathcal{F}^{t-1}\right]$. Consequently, consider
\begin{align}
\label{eq:min_imp}
    \mathbb{E}_{P_{t}}\left[\left.\left\| \bar{g}_{t}^{\mathcal{I}^{t}}\right\|^2 \right|\mathcal{F}^{t-1}\right] &=\mathbb{E}_{P_{t}}\left[\left.\left\| \bar{g}_{t}^{\mathcal{I}^{t}}-g_t+g_t\right\|^2 \right|\mathcal{F}^{t-1}\right]\nonumber\\
    &=\mathbb{E}_{P_{t}}\left[\left.\left\| \bar{g}_{t}^{\mathcal{I}^{t}}-g_t\right\|^2 \right|\mathcal{F}^{t-1}\right]+
    \mathbb{E}_{P_{t}}\left[\left. 2<\bar{g}_{t}^{\mathcal{I}^{t}}-g_t,g_t> \right|\mathcal{F}^{t-1}\right]+\left\|g_t\right\|^2    \nonumber\\
    &=\mathbb{E}_{P_{t}}\left[\left.\left\| \bar{g}_{t}^{\mathcal{I}^{t}}-g_t\right\|^2 \right|\mathcal{F}^{t-1}\right]+2<\mathbb{E}_{P_{t}}\left[\left.\bar{g}_{t}^{\mathcal{I}^{t}}\right|\mathcal{F}^{t-1}\right]-g_t,g_t> +\left\|g_t\right\|^2    \nonumber\\
    &=\mathbb{E}_{P_{t}}\left[\left.\left\| \bar{g}_{t}^{\mathcal{I}^{t}}-g_t\right\|^2 \right|\mathcal{F}^{t-1}\right]+\left\|g_t\right\|^2,
\end{align}
where the last equality holds due to (\ref{eq:importance grad}). Continuing, we have
{\footnotesize
\begin{align}
    &\mathbb{E}_{P_{t}}\left[\left.\left\| \bar{g}_{t}^{\mathcal{I}^{t}}-g_t\right\|^2 \right|\mathcal{F}^{t-1}\right]= \mathbb{E}_{P_{t}}\left[\left.\left\| \frac{1}{\left|\mathcal{I}^{t}\right|}\sum_{i\in\mathcal{I}^{t}}\left(\frac{1}{\left|\mathcal{D}\right|P_t^i} g_{t}^i-g_t\right)\right\|^2 \right|\mathcal{F}^{t-1}\right]\nonumber\\
    &=\frac{1}{\left|\mathcal{I}^{t}\right|^2}\mathbb{E}_{P_{t}}\left[\left.\left\| \sum_{i\in\mathcal{I}^{t}}\left(\frac{1}{\left|\mathcal{D}\right|P_t^i} g_{t}^i-g_t\right)
    \right\|^2 \right|\mathcal{F}^{t-1}\right]\nonumber\\
    &=\frac{1}{\left|\mathcal{I}^{t}\right|^2}\mathbb{E}_{P_{t}}\left[\left.\sum_{i\in\mathcal{I}^{t}}\left\| \frac{1}{\left|\mathcal{D}\right|P_t^i} g_{t}^i-g_t\right\|^2+\sum_{(i,j)\in\mathcal{I}^{t},i\neq j}<\frac{1}{\left|\mathcal{D}\right|P_t^i} g_{t}^i-g_t, \frac{1}{\left|\mathcal{D}\right|P_t^j} g_{t}^j-g_t> \right|\mathcal{F}^{t-1}\right]\nonumber\\
    &=\frac{1}{\left|\mathcal{I}^{t}\right|^2}\left(\mathbb{E}_{P_{t}}\left[\left.\sum_{i\in\mathcal{I}^{t}}\left\| \frac{1}{\left|\mathcal{D}\right|P_t^i} g_{t}^i-g_t\right\|^2\right|\mathcal{F}^{t-1}\right]+\mathbb{E}_{P_{t}}\left[\left. \sum_{(i,j)\in\mathcal{I}^{t},i\neq j}<\frac{1}{\left|\mathcal{D}\right|P_t^i} g_{t}^i-g_t, \frac{1}{\left|\mathcal{D}\right|P_t^j} g_{t}^j-g_t> \right|\mathcal{F}^{t-1}\right]\right)\nonumber\\
    &=\frac{|\mathcal{I}^{t}|}{\left|\mathcal{I}^{t}\right|^2}\mathbb{E}_{P_{t}}\left[\left.\left\| G_1-g_t\right\|^2\right|\mathcal{F}^{t-1}\right]+\frac{2}{\left|\mathcal{I}^{t}\right|^2}\binom{\left|\mathcal{I}^{t}\right|}{2} <\mathbb{E}_{P_{t}}\left[\left.G_1-g_t\right|\mathcal{F}^{t-1}\right], \mathbb{E}_{P_{t}}\left[\left.G_2-g_t\right|\mathcal{F}^{t-1}\right]> \nonumber\\
    &= \frac{1}{\left|\mathcal{I}^{t}\right|}\mathbb{E}_{P_{t}}\left[\left.\left\| G_1-g_t
    \right\|^2 \right|\mathcal{F}^{t-1}\right]\nonumber\\
    &=\frac{1}{\left|\mathcal{I}^{t}\right|}\left(\mathbb{E}_{P_{t}}\left[\left.\left\| G_1
    \right\|^2 \right|\mathcal{F}^{t-1}\right]-\left\|g_t\right\|^2\right).
    \label{eq:final_min}
\end{align}
}
The fifth equality holds since $G_i$ and $G_j$ are independent, moreover, the seventh equality is valid due to (\ref{eq:importance grad}). Inserting (\ref{eq:final_min}) into (\ref{eq:min_imp}) yields
\begin{align}
    \mathbb{E}_{P_{t}}\left[\left.\left\| \bar{g}_{t}^{\mathcal{I}^{t}}\right\|^2 \right|\mathcal{F}^{t-1}\right] = \frac{1}{\left|\mathcal{I}^{t}\right|\left|\mathcal{D}\right|^2}\sum_{(x_i,z_i)\in\mathcal{D}}\frac{1}{P_t^i}\left\| g_{t}^i\right\|^2 -\frac{1}{\left|\mathcal{I}^{t}\right|}\left\|g_t\right\|^2 + \left\|g_t\right\|^2.
\label{eq:final_goal_imp}
\end{align}
As (\ref{eq:final_goal_imp}) shows, maximizing $\mathbb{E}_{P_{t}}\left[ \Delta^{(t)} \right]$ is equivalent to minimizing $\frac{1}{P_t^i}\left\| g_{t}^i\right\|^2$. By using the Jensen's inequality, it follows that
\begin{align}
    \sum_{(x_i,z_i)\in\mathcal{D}}\frac{1}{P_t^i}\left\| g_{t}^i\right\|^2=\sum_{(x_i,z_i)\in\mathcal{D}}P_t^i\left( \frac{\left\|g_{t}^i\right\|}{P_t^i}\right)^2\geq \left(\sum_{(x_i,z_i)\in\mathcal{D}} \left\|g_{t}^i\right\|\right)^2,
\end{align}
and the equality holds when $P_t^i=\left\|g_{t}^i\right\|/\sum_{(x_j,z_j)\in\mathcal{D}}\left\|g_{t}^j\right\|$. Note that $g_{t}^i=\left.\frac{\partial L(z_i,f_{t-1}(x_i)+\epsilon g(x_i)}{\partial g}\right|_{\epsilon =0}$ is proportional to the boosting weights $w_{t}(x_i,z_i)$ of sample $(x_i,z_i)$ as stated in (\ref{eq:grad-weight}), therefore, the claim in (\ref{importance sample distribution}) follows.
\end{proof} 

\end{document}